\documentclass[11pt,letterpaper]{article}
\usepackage[T1]{fontenc}
\usepackage[utf8]{inputenc}
\usepackage{lmodern}
\usepackage[margin=1in]{geometry}
\usepackage{amsmath,amssymb,amsthm,bm,graphicx,booktabs,microtype}
\usepackage{xcolor}
\usepackage{placeins}
\usepackage[hidelinks]{hyperref}
\hypersetup{
  pdftitle={Low-Bit Recurrent States in Hybrid Language Models},
  pdfauthor={HONGREN CHEN; JIAYANG HE},
  pdfsubject={Author preprint of the manuscript submitted to ICASSP 2027},
  pdfkeywords={linear attention, delta rule, state-space models, quantization, bit allocation},
  pdfdisplaydoctitle=true
}
\theoremstyle{plain}
\newtheorem{proposition}{Proposition}

\makeatletter
\renewcommand{\section}{\@startsection{section}{1}{\z@}%
  {-2.3ex plus -0.5ex minus -0.2ex}{1.0ex plus 0.2ex}%
  {\normalfont\large\bfseries}}
\long\def\@makecaption#1#2{%
  \vskip\abovecaptionskip
  {\small\noindent\textbf{#1.} #2\par}%
  \vskip\belowcaptionskip}
\makeatother
\newcommand{\bk}{\bm{k}}
\newcommand{\ba}{\bm{a}}
\newcommand{\bv}{\bm{v}}
\newcommand{\bw}{\bm{w}}
\newcommand{\bq}{\bm{q}}
\newcommand{\bs}{\bm{s}}
\newcommand{\Diag}{\operatorname{diag}}
\newcommand{\expmone}{\operatorname{expm1}}
\newcommand{\GM}{\operatorname{GM}}
\newcommand{\AM}{\operatorname{AM}}
\newcommand{\E}{\mathbb{E}}

\title{\fontsize{14}{17}\selectfont\bfseries Low-Bit Recurrent States in Hybrid Language Models}
\author{%
  HONGREN CHEN\\[4pt]
  \small\href{mailto:hungjenchen@telswarm.com}{hungjenchen@telswarm.com}
  \and
  JIAYANG HE\\[4pt]
  \small\href{mailto:1755147627@qq.com}{1755147627@qq.com}}
\date{\small Author preprint -- September 25, 2026}
\begin{document}
\maketitle
\begin{center}
\small Submitted to ICASSP 2027. Not yet peer reviewed.\par
\end{center}
\begin{abstract}
Hybrid language models maintain fixed-size recurrent states, but existing quantizers typically
use eight bits or more. Quantization errors persist according to channel decay rates.
We derive distortion weights from the observability Gramian and combine them with normalized
state ranges for mixed-precision bit allocation, without calibration data, rotation, or training.
We also quantize decay rates logarithmically. With per-token state quantization, a four-bit mean
payload reduces excess negative log-likelihood by factors of $3.3$--$27.9$ relative to the best
of seven baselines across three hybrid models; metadata costs vary. At six bits, negative
log-likelihood differs from the FP32-state baseline by less than $0.005$ nats. Ablations separate
gains from variable bit widths, decay weighting, and range normalization. With less frequent
write-backs, gains diminish and depend on the model and budget.
\end{abstract}

\par\noindent\textbf{Keywords:}
linear attention, delta rule, state-space models, quantization, bit allocation
\par\medskip

\FloatBarrier
\section{Introduction}
\label{sec:intro}

Gated linear recurrences retain information in fixed-size states~\cite{mamba,mamba2,gdn}.
Kimi Delta Attention (KDA) multiplies row $i$ by $e^{-s_i}$ per token, with time constant
$1/s_i$ for constant $s_i$~\cite{kimilinear}. Gated DeltaNet and Mamba-2 share decay within
each head~\cite{gdn,mamba2}; Gated DeltaNet-2 separates erasing and writing~\cite{gdn2}.
Recurrent states commonly remain FP32~\cite{damp}, even when projections use
NVFP4~\cite{nvfp4gdn}; Nemotron~3 uses FP16 states with stochastic rounding~\cite{nemotron3}.

Range-based quantization controls error at a single write-back, as in KV-cache
compression~\cite{kvquant}. Recurrent-state errors also affect later readouts: rapidly decaying
rows dissipate errors, whereas slowly decaying rows accumulate them across write-backs.

Under a linear noise model, an observability Gramian determines accumulated readout distortion.
For diagonal decay, its diagonal entries sum query-weighted squared decay factors
(Fig.~\ref{fig:mech}a). Combining these with normalized state ranges yields a classical
high-rate bit-allocation rule~\cite{huang63,gersho} without calibration data or training.
We isolate decay-weighting gains under frequent write-back on three hybrid models, using
excess negative log-likelihood (NLL).

\noindent\textbf{Related work.} DAMP~\cite{damp} multiplies each key row's low-precision error
energy by a capped version of $w^{(1)}=1/(1-a^2)$, with effective decay $a$.
Retaining $16$ of $128$ rows per head in FP16 and the rest in Hadamard INT8 requires
$9.875$ bits per element; INT4 and NVFP4 severely degrade accuracy. DASC~\cite{dasc} omits
short-horizon rows from prefix checkpoints; deferred write-back reduces memory
traffic~\cite{deltalog,kvbuffer,nemotron3ultra}. We adapt recurrent-state
quantizers~\cite{quamba2,ssdi8,qmamba,nemotron3}, KV-cache
quantizers~\cite{kivi,kvquant,quarot,turboquant,kvarn,oscar}, and four-bit
formats~\cite{ocpmx,nvfp4gdn,fourosix} as baselines.

Persistence weighting~\cite{damp} and rate--distortion allocation~\cite{aatc,kvtc} are
established. We connect persistence to readout distortion, normalize ranges within output
normalization groups, and evaluate lower-precision variable-width allocation and logarithmic
decay-rate quantization. Without erasure, our weight is DAMP's uncapped persistence times
$e^{-2s_i}$; their held-out NLL differs by at most $0.003$ nats with chunked write-back.

\begin{figure}[tbp]
\centering
\includegraphics[width=\textwidth]{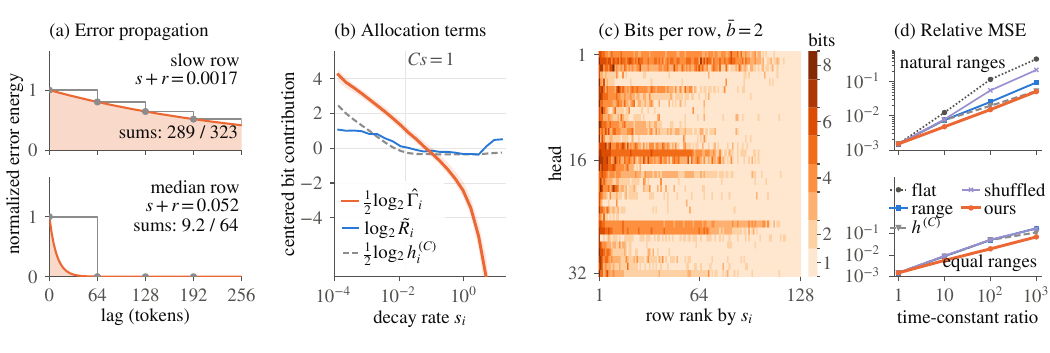}
\caption{\textbf{Error propagation and bit allocation.} (a)~Squared decay of an injected state
error in two Kimi-Linear rows; gray steps sample it at write-backs. Paired sums refer to the decay curve and steps. (b)~Centered contributions to
the bit-allocation target. (c)~Allocated bit widths in one layer. (d)~Relative readout MSE in a
synthetic recurrence with natural or equal row ranges. Settings: $C=64$, two-bit mean payload.}
\label{fig:mech}
\end{figure}

\FloatBarrier
\section{Propagation of state quantization error}
\label{sec:theory}

With $\|\bk_t\|_2=1$, $\mathbf D_t=\Diag(e^{-\bs_t})$, and $\bs_t>\bm 0$, each head's
state $\mathbf S_t\in\mathbb R^{d_k\times d_v}$ follows
\begin{equation}
\begin{aligned}
\mathbf S_t&=(\mathbf I-\bk_t\ba_t^{\!\top})\mathbf D_t\mathbf S_{t-1}
+\bk_t(\bw_t\odot\bv_t)^{\!\top},\\
\bm o_t&=\mathbf S_t^{\!\top}\bq_t,
\end{aligned}
\label{eq:rule}
\end{equation}
where $\ba_t$ and $\bw_t$ control erasure and writing. KDA and Gated DeltaNet use
$\ba_t=\beta_t\bk_t$, $\bw_t=\beta_t\bm 1$~\cite{kimilinear,gdn}; Mamba-2 has no
erasure~\cite{mamba2}.

\noindent\textbf{Accumulated readout error.} Write-back occurs every $C$ tokens: once per
prefill chunk, per token in standard decoding, or less frequently with deferred write-back
($C=8$ in~\cite{nemotron3ultra}). Outputs precede quantization, so an error introduced at
$\tau$ is first read at $\tau+1$ and propagates through
$\mathbf F_t=(\mathbf I-\bk_t\ba_t^{\!\top})\mathbf D_t$ between write-backs.

\begin{proposition}[Time-averaged readout distortion]
\label{prop:gram}
Assume stationary layer inputs independent of quantization noise and finite accumulated error
energy. Conditional on these inputs, noise is zero-mean, uncorrelated across rows and write-backs,
with per-element row variance $\sigma_i^2$. Steady-state readout MSE per value column,
averaged over positions within a write-back interval, is $\sum_i\Gamma_i\sigma_i^2$, where
\begin{equation}
\Gamma_i=\frac1C\sum_{\ell\ge1}\E\big[(\Phi_\ell^{\!\top}\bq_{\tau+\ell})_i^2\big],
\label{eq:gram}
\end{equation}
$\Phi_\ell=\mathbf F_{\tau+\ell}\cdots\mathbf F_{\tau+1}$, and $C\Gamma_i$ is a diagonal
entry of the expected observability Gramian. Constant diagonal decay without erasure gives
$\Gamma_i=\rho_i/(C\expmone(2s_i))$, with $\rho_i=\E q_i^2$ and $\expmone(x)=e^x-1$.
For constant per-head decay and $\ba=\beta\bk$, $\beta\in[0,2]$,
$\Gamma_i\le\E\|\bq\|_2^2/(C\expmone(2s))$.
\end{proposition}
\noindent\emph{Proof sketch.} Cross terms vanish. Each write-back contributes at all lags
$\ell\ge1$, once per $C$ tokens. Apply $\sum_{\ell\ge1}e^{-2s\ell}=1/\expmone(2s)$ and
$\|\mathbf I-\beta\bk\bk^{\!\top}\|_2\le1$. $\square$
The factor $1/C$ scales distortion but leaves allocation unchanged. Exogenous inputs are a local approximation: teacher forcing does not
eliminate cross-layer noise propagation.

\noindent\textbf{Correlated rounding errors.} Deterministic rounding at $C=1$ can suppress
updates below half a quantization step. Simulations with Kimi decays, rank-one writes, and
$2$K tokens show lag-one error correlations up to $0.23$ for $s_i\le10^{-3}$ at three to five
bits. Readout error is $0.17$--$2.4$ times that of independent noise with variance
$\Delta^2/12$ (step size $\Delta$), with discrepancies within $20\%$ from six bits.
Proposition~\ref{prop:gram} is therefore a decoding heuristic: in these simulations, the model
overestimates the gain over uniform four-bit allocation by $1.7$--$2.6$ times.

\noindent\textbf{Approximate erasure rate.} For channel-wise decay, we approximate
$\Gamma_i\approx\rho_i/[C\expmone(2(s_i+r_i))]$, with
$r_i=\frac12\E[\beta(2-\beta)k_i^2]$. This is not a bound: errors orthogonal to a repeated
key are never erased. With a fixed key, $\ba=\bk/2$, $s=0.01$, and $C=1$, the exact and
approximate $\Gamma_i$ are $24.9$ and $2.06$. For simulated $s$ log-uniform in $[10^{-3},1]$,
allocation distortion differs from the exact-Gramian allocation by at most $3\%$ with diverse
keys and $17$--$23\%$ with persistent keys. Omitting $r_i$ gives comparable held-out NLL;
we retain it.

\noindent\textbf{Range normalization.} To first order, RMSNorm projects out radial perturbations
and divides by the group's output RMS. Ignoring the projection and learned gains motivates
normalizing ranges within each output normalization group: one head for KDA and Gated DeltaNet,
or $12$ heads for Nemotron-3-Nano's Mamba-2 layers~\cite{nemotronelastic}. Assuming uniform
$\rho_i$ and output mean square proportional to the group's mean squared row range gives the
proxy $\Gamma_i\tilde R_i^2$. We evaluate these assumptions through ablations
(Sec.~\ref{sec:ablate}); they are not exact RMSNorm properties.

\FloatBarrier
\section{Decay-aware quantization}
\label{sec:method}

\noindent\textbf{Decay-gate quantization.} A rate error $\Delta s$ changes the signal retained
at the time constant $1/s$ by $e^{-\Delta s/s}$. The decay weight's log-sensitivity also tends
to $-1$ as $s\to0$, motivating levels uniform in $\log s$ per channel. At two bits, gate-only
quantization raises WikiText perplexity by at most $0.27$ on Qwen3.5-0.8B/2B/9B, Qwen3.8-27B,
and Kimi. It outperforms levels uniform in $g=-s$ on all models; the latter give perplexity
$86.9$ on Kimi even at four bits, versus $9.33$ in full precision. Subsequent experiments keep
gates in full precision. With both gates and states at two bits, held-out NLL increases are
approximately additive ($C=64$: gates $+0.037$, states $+0.006$, both $+0.047$).

\noindent\textbf{State bit allocation.} We minimize the high-rate proxy $\sum_i c_i2^{-2b_i}$ using
\begin{equation}
c_i=\hat\Gamma_i\tilde R_i^2,\qquad
\hat\Gamma_i=\frac1{\expmone(2(\hat s_i+r_i))}.
\label{eq:weight}
\end{equation}
Here $\tilde R_i=R_i/\operatorname{RMS}_{j\in u(i)}R_j$ normalizes ranges within output group
$u(i)$; $R_i$ is the row absolute maximum, averaged over rows for a head unit. At layer-wise
mean budget $\bar b$, the unconstrained real-valued solution is
\begin{equation}
b_i=\bar b+\tfrac12\log_2\big(c_i/\GM(c)\big).
\label{eq:alloc}
\end{equation}
Under this proxy, the distortion reduction over uniform allocation is $\AM(c)/\GM(c)$,
where $\AM$ and $\GM$ are the arithmetic and geometric means~\cite{huang63,gersho}. This
optimality does not extend to low-bit integer quantizers. We round, clip to $[b_{\min},b_{\max}]$,
and adjust widths to meet the layer budget, with $b_{\max}=8$ ($16$ at $C=1$). At each write-back,
$\hat s_i=-\frac12\log\operatorname{EMA}_{64}(e^{2g_{t,i}})$ and
$r_i=\frac12\operatorname{EMA}_{64}[\beta_t(2-\beta_t)k_{t,i}^2]$, using a chunk-aware exponential
moving average with target span $64$ tokens. In simulations with log-normal gates and widths
in $[1,8]$, this increases readout error by at most $2.5\%$ over exact-Gramian allocation across
tested $C$. A one-token estimate gives comparable decoding NLL.

\noindent\textbf{Allocation units.} For channel-wise decay (KDA) or normalization spanning heads
(Mamba-2), each row has a scale and width, with $b_{\min}=1$. When decay and normalization
are per head (Gated DeltaNet), each head has one width and scales per block of $32$ key
channels~\cite{kivi}. Then $\tilde R=1$ and only decay determines widths. We use $b_{\min}=2$
to avoid one-bit distortion of an entire head; a two-bit mean reduces to uniform block quantization.

\noindent\textbf{Levels and storage.} For $b\ge2$, we use levels
$\{-n,\ldots,n\}\,\Delta$, with $n=2^{b-1}-1$ and $\Delta=fR/n$. We fit $f$ over $16$ values in
$[1/4,1]$ by squared error per scale group, storing the result in an FP16 scale. One bit stores
the sign times its mean absolute value. We count ternary levels as two bits, without
ternary packing. An FP16 scale and three-bit width add $19/d_v$ bits per row element
($20/d_v$ with four-bit widths at $C=1$). Block scales add $0.5$ bits per element; per-head
widths are also counted. Prefix sums of widths give row offsets.

\FloatBarrier
\section{Bit-allocation ablations}
\label{sec:ablate}

\begin{table}[tbp]
\centering\small
\caption{\textbf{Held-out bit-allocation ablations} on Kimi (excess NLL, nats). Only $c_i$ varies. Payload: $4$ bits at $C=1$, $2$ bits otherwise. Settings: $2$K-token windows ($8$ at $C=1$, else $32$). Bold: minimum and entries whose paired $95\%$ interval against it includes zero; --: not evaluated.}
\label{tab:ablate}
\vspace{1pt}
\setlength{\tabcolsep}{6pt}
\begin{tabular}{@{}lccc@{}}
\toprule
\emph{weight $c_i$} & $C=1$ & $C=16$ & $C=64$\\
\midrule
$R_i^2$, range alone & +0.047 & +0.030 & +0.009\\
$\hat\Gamma_iR_i^2$, raw range & -- & +0.034 & +0.012\\
\midrule
$\tilde R_i^2$, relative range & +0.072 & +0.043 & +0.011\\
$w^{(1)}_i\tilde R_i^2$, persistence & -- & \textbf{+0.023} & +0.009\\
$h^{(C)}_i\tilde R_i^2$, write-back factor & -- & \textbf{+0.019} & \textbf{+0.002}\\
$\hat\Gamma_i\tilde R_i^2$, erase $r_i=0$ & -- & \textbf{+0.021} & +0.006\\
$\hat\Gamma_{\pi(i)}\tilde R_i^2$, shuffled decay & -- & +0.082 & +0.045\\
$\hat\Gamma_i\tilde R_i^2$ (ours) & \textbf{+0.034} & \textbf{+0.022} & +0.006\\
\midrule
uniform bit widths & -- & +0.167 & +0.081\\
\bottomrule
\end{tabular}
\end{table}

\begin{figure}[tbp]
\centering
\includegraphics[width=\textwidth]{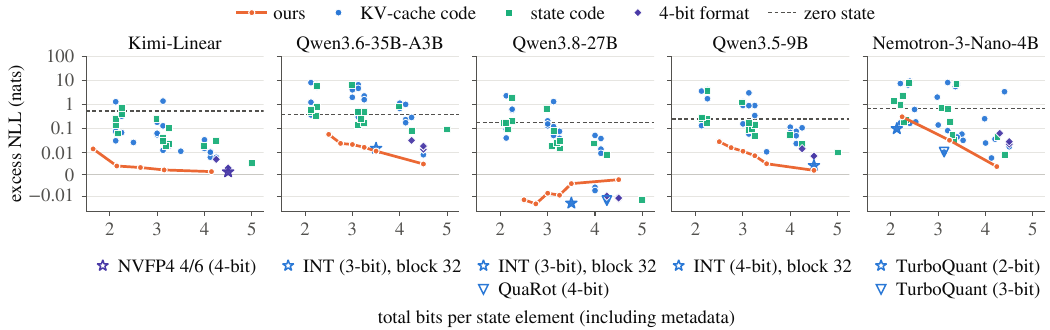}
\caption{\textbf{Chunked state quantization} (WikiText, $C=64$). Excess NLL is relative to FP32
on a symmetric logarithmic scale. Stars and triangles identify configurations in the legends
below each panel (nominal payloads).}
\label{fig:front}
\end{figure}

We use expert-pruned~\cite{reap} Kimi-Linear-REAP-35B-A3B-Instruct ($20$ KDA
layers~\cite{kimilinear}), fake-quantized at each write-back in the released kernel. Unless
specified, the mean payload is two bits. Excess NLL is measured against FP32 states on
WikiText~\cite{wikitext} in $2$K-token windows. The rule was selected on the first $32$K tokens
at $C=64$; Table~\ref{tab:ablate} uses tokens from $128$K onward, and a second held-out set
spans $32$K--$64$K. Intervals are paired $95\%$ bootstraps of our NLL minus the alternative's;
``tie'' means the interval includes zero.

\noindent\textbf{Effect of decay weighting.} Relative to allocation by raw range alone, our method
reduces excess NLL by $0.013$ at $C=1$ ($[-0.021,-0.006]$), $0.008$ at $C=16$
($[-0.013,-0.002]$), and $0.003$ at $C=64$ ($[-0.007,+0.000]$, a tie). Relative to normalized
range alone, the reductions are $0.038$, $0.021$, and $0.005$; the last comparison ties on the
second held-out set. This agrees qualitatively with $1/C$ in~\eqref{eq:gram}. The observed gain ratio between $C=16$ and $64$ is $4.0$ versus normalized range but $1.9$
versus uniform widths; its uncertainty is large (ratios $2$--$35$), and $C=1$ uses four bits. Shuffling decay weights across rows worsens NLL in all $32$ windows at both $C=16$
and $64$, confirming the importance of matching weights to rows.

\noindent\textbf{Effect of range normalization.} Decay-only allocation is worse than ours on Kimi
($-0.059$, $[-0.072,-0.046]$ on the second set at $C=64$). Normalization preserves the ranking
of rows within a head but changes allocation across heads. Slowly decaying heads tend to have
larger ranges: the median layer-wise Spearman correlation between a head's mean $s_i$ and RMS
range is $-0.65$ (range $-0.74$ to $-0.18$). Normalized range alone removes this indirect decay
signal and is $0.013$ worse than raw range at $C=16$. Combining decay with raw range can
overweight these heads: it is worse than ours at both $C=16$ and $64$ ($-0.012$,
$[-0.016,-0.008]$ at $C=16$). On Nemotron, where normalization spans $12$ heads, two-bit excess
NLL is $1.92$ with decay alone, $8.33$ with raw range, $3.20$ with normalized range, and $0.306$
with the combined weight.

\noindent\textbf{Per-head allocation.} On Qwen3.6-35B-A3B~\cite{qwen36}, a one-bit floor and
two-bit mean give excess NLL $1.42$ with raw-range allocation, exceeding the zero-state baseline,
and $0.219$ with decay allocation, which is lower in all $16$ windows. Both exceed the uniform
block quantizer's $0.057$, motivating the two-bit floor. During decoding, decay weighting
outperforms range-only allocation (Table~\ref{tab:decode}). On Qwen3.5-9B at $C=64$, decay
allocation has lower NLL at all four budgets between $2.25$ and $3$ bits; the paired interval
excludes zero at three bits.

\noindent\textbf{Choice of persistence factor.} With normalized ranges, both DAMP's uncapped
$w_i^{(1)}$ and the write-back-time factor $h_i^{(C)}=1/(1-e^{-2Cs_i})$ give NLL within $0.004$
of ours. Both tie ours at $C=16$ and improve on normalized range alone by $0.020$--$0.024$.
At $C=64$, $w_i^{(1)}$ is worse ($-0.003$, $[-0.005,-0.001]$), whereas $h_i^{(C)}$ is better
($+0.004$, $[+0.0002,+0.008]$). Omitting the erasure rate ties at both intervals.

\FloatBarrier
\section{Experimental comparisons}
\label{sec:exp}

\begin{table}[tbp]
\centering\small
\setlength{\tabcolsep}{6pt}
\renewcommand{\arraystretch}{1.12}
\caption{\textbf{Per-token state quantization} ($C=1$; excess NLL, nats). Ours outperforms all seven baselines in every window at four bits (unadjusted sign test $p=0.004$ per comparison). Settings: $8$ WikiText windows of $2$K tokens, $b_{\max}=16$. INT uses fitted scales; range-only uses the same units as ours. Total bits: Kimi ($^\dagger$$4.5$ on Qwen). Bold: minimum and statistical ties.}
\label{tab:decode}
\vspace{1pt}
\begin{tabular}{@{}lcccc@{}}
\toprule
\emph{(a) $\bar b=4$} & bits & Kimi & Qwen3.8 & Qwen3.6\\
\midrule
INT, row & $4.125$ & +0.781 & +1.47 & +4.14\\
INT, block $32$ & $4.5$ & +0.252 & +0.293 & +2.41\\
TurboQuant~\cite{turboquant} & $4.125$ & +0.131 & +0.073 & +2.83\\
Hadamard INT, block $32$~\cite{damp} & $5.0$ & +0.135 & +0.142 & +4.44\\
NVFP4~\cite{nvfp4gdn} & $4.5$ & +0.236 & +0.242 & +2.62\\
DAMP-style (INT3/8)~\cite{damp} & $5.016$ & +0.222 & +1.79 & +8.68\\
Stochastic rounding~\cite{nemotron3} & $4.25$ & +0.663 & +1.73 & +5.41\\
\midrule
range only & $4.156^\dagger$ & +0.047 & \textbf{+0.024} & +1.15\\
\textbf{ours} & $4.156^\dagger$ & \textbf{+0.040} & \textbf{+0.015} & \textbf{+0.086}\\
\bottomrule
\end{tabular}

\vspace{2pt}
\setlength{\tabcolsep}{6pt}
\begin{tabular}{@{}lcccccc@{}}
\toprule
 & \multicolumn{2}{c}{Kimi} & \multicolumn{2}{c}{Qwen3.8} & \multicolumn{2}{c}{Qwen3.6}\\
\cmidrule(lr){2-3}\cmidrule(lr){4-5}\cmidrule(l){6-7}
\emph{(b) $\bar b$} & $5$ & $6$ & $5$ & $6$ & $5$ & $6$\\
\midrule
INT, row & +0.174 & +0.027 & +0.571 & +0.338 & +2.03 & +1.22\\
INT, block $32$ & +0.075 & +0.034 & +0.010 & \textbf{$-$0.010} & +0.136 & +0.028\\
\textbf{ours} & \textbf{+0.014} & \textbf{+0.003} & \textbf{$-$0.001} & \textbf{$-$0.004} & \textbf{+0.021} & \textbf{+0.001}\\
\bottomrule
\end{tabular}
\end{table}

\begin{table}[tbp]
\centering\small
\setlength{\tabcolsep}{6pt}
\renewcommand{\arraystretch}{1.12}
\caption{\textbf{Chunked state quantization} ($C=64$; excess NLL, nats). Total bits are shown at $\bar b=2$ ($^\dagger$Qwen3.6: DASC $2.234$, ours $2.5$). Bold: minimum and statistical ties. INT scales are fitted; Quamba2 uses static scales. Requirements: calibration (C), FP16 values (F), rotation (R). The zero-state baseline resets the state every $64$ tokens.}
\label{tab:state}
\vspace{1pt}
\begin{tabular}{@{}lcccccc@{}}
\toprule
 & & & \multicolumn{2}{c}{Kimi-Linear} & \multicolumn{2}{c}{Qwen3.6}\\
\cmidrule(lr){4-5}\cmidrule(l){6-7}
code & req. & bits & $\bar b{=}2$ & $3$ & $\bar b{=}2$ & $3$\\
\midrule
zero state & -- & $0$ & \multicolumn{2}{c}{+0.525} & \multicolumn{2}{c}{+0.373}\\
\multicolumn{7}{@{}l}{\emph{KV-cache codes}}\\
INT, row & -- & $2.125$ & +0.076 & +0.034 & +0.365 & +0.299\\
INT, block $32$ & -- & $2.5$ & +0.026 & +0.011 & \textbf{+0.057} & \textbf{+0.015}\\
KIVI, block $32$~\cite{kivi} & -- & $3.0$ & +0.182 & +0.034 & +4.85 & +1.12\\
KVQuant, $1\%$~\cite{kvquant} & C\,F & $3.0$ & +0.055 & +0.017 & +1.96 & +0.761\\
TurboQuant~\cite{turboquant} & R & $2.125$ & +0.031 & +0.012 & +1.22 & +0.516\\
QuaRot~\cite{quarot} & R & $2.25$ & +0.068 & +0.019 & +6.27 & +2.23\\
OScaR~\cite{oscar} & R & $3.125$ & +1.37 & +0.010 & +4.60 & +0.997\\
\multicolumn{7}{@{}l}{\emph{state codes}}\\
Quamba2~\cite{quamba2} & C & $2.125$ & +0.252 & +0.041 & +0.466 & +0.140\\
DAMP-style~\cite{damp} & C\,R & $2.984$ & +0.249 & +0.019 & +6.32 & +0.825\\
DASC-style~\cite{dasc} & -- & $2.167^\dagger$ & +0.062 & +0.019 & +0.332 & +0.250\\
\midrule
\textbf{ours} & -- & $2.148^\dagger$ & \textbf{+0.003} & \textbf{+0.002} & \textbf{+0.057} & \textbf{+0.012}\\
\bottomrule
\end{tabular}
\end{table}

\noindent\textbf{Protocol.} Each write-back is quantized and dequantized before reuse; no
unquantized residual state is retained. Token-wise evaluation ($C=1$) uses Kimi,
Qwen3.8-27B~\cite{qwen38}, and Qwen3.6-35B-A3B over eight $2$K-token WikiText windows,
overlapping the $C=64$ rule-selection data. Reference runs use $C=64$ and $32$K tokens,
adding Qwen3.5-9B~\cite{qwen35} and Nemotron-3-Nano-4B ($1$K windows). Bits include scales,
zero points, outlier indices, and width maps~\cite{kvquant,damp}. TurboQuant is the best tested
variant. DAMP-style uses DAMP's score to assign INT8 to top-ranked rows within budget;
the rest use Hadamard INT with $\lfloor\bar b\rfloor-1$ bits. Eight decoding windows limit
bootstrap reliability. Quantizers are simulated, without kernel timings.

\noindent\textbf{Per-token state quantization} (Table~\ref{tab:decode}). At a four-bit mean,
ours outperforms all seven baselines in every window, reducing excess NLL by factors of
$3.3$, $4.8$, and $27.9$ relative to the best baseline on Kimi, Qwen3.8, and Qwen3.6, respectively. On Qwen3.8, TurboQuant uses $0.375$ fewer bits; at equal storage,
NVFP4 has $16$ times our excess NLL. Range-only allocation with identical units and limits
is $0.007$ worse on Kimi ($0.013$ held out) and ties on Qwen3.8. On Qwen3.6, all four-bit
baselines and range-only allocation ($+1.15$) exceed the zero-state NLL increase
($+0.373$, resets every $64$ tokens), versus our $+0.086$. DAMP-style's low tier is INT3,
rather than DAMP's reported INT4 or INT8. At five and six bits, ours outperforms both
uniform-width quantizers (one tie); six-bit excess NLL is at most $0.0034$.
DAMP's published setting uses $9.875$ bits. At a four-bit mean, Kimi's $C=64$ statistics
assign $4\%$ of rows more than eight bits, up to $14$.

\noindent\textbf{Chunked state quantization} (Fig.~\ref{fig:front}, Table~\ref{tab:state}).
For channel-wise decay, ours achieves the lowest measured NLL at $1.6$--$4.1$ total bits
per element and ties NVFP4 Four Over Six at a four-bit payload using $0.35$ fewer total bits. For per-head decay at three and four bits, it leads on Qwen3.5-9B
and ties the block quantizer on Qwen3.6. At four bits, it outperforms all tested four-bit formats
on Qwen3.6, Qwen3.5-9B, and Nemotron.

\noindent\textbf{Long-context retrieval.} With a two-bit mean and $b_{\max}=8$,
write-backs occur every $64$ prompt tokens and every answer token. On RULER~\cite{ruler} at
$16$K context (Kimi; $20$ prompts per task), all quantizers and FP32 solve every multi-key query. For multi-value queries (four values per needle), recovery is $0.90$ for ours, $0.80$
for row-wise INT and TurboQuant, $0.76$ for the fitted block, and $1.00$ for FP32. Paired gains
are $0.10$ over INT and TurboQuant ($[0.05,0.15]$; no worse on any prompt), $0.14$ over the
$2.5$-bit block ($[0.08,0.20]$), and $-0.10$ against FP32 ($[-0.15,-0.05]$).

\noindent\textbf{Limitations.} At $C=64$, ours is best in four of $15$ model--budget pairs,
tied in seven (three identical to the block quantizer), and worse in four. On Qwen3.8, some
quantizers have lower NLL than FP32: the three-bit block outperforms ours ($+0.016$,
$[+0.003,+0.030]$), and at four bits five quantizers are better and one ties. Decay also loses to
range-only allocation at two of four budgets. On Nemotron, nine quantizers are better at two bits
(TurboQuant $+0.098$, ours $+0.306$); Hadamard TurboQuant is better at three ($+0.011$, ours
$+0.033$). On held-out Kimi data at $C=64$ and three bits, normalized range outperforms ours
($+0.009$, $[+0.003,+0.015]$) and FP32, which the output-MSE model~\eqref{eq:gram} cannot
explain. At two bits, $h_i^{(C)}$ is marginally better than ours (Table~\ref{tab:ablate}).

\FloatBarrier
\section{Conclusion}
\label{sec:conclusion}
Decay weighting with normalized ranges improves low-bit recurrent-state quantization, especially
with frequent write-backs. Correlated rounding and implementation costs require further study.

\FloatBarrier
\section{Acknowledgments}
OpenAI Codex and Anthropic Claude Code assisted with language editing throughout the manuscript,
research-code development, and idea discussions. The human authors originated the core claims,
understand and have reviewed all content, and take full responsibility for it.

\FloatBarrier
\section{Compliance with Ethical Standards}
This study comprises numerical simulations and evaluations of released models on public text
benchmarks. It involves no human participants or animal experiments; ethical approval was not required.

\FloatBarrier
\section{Funding and Competing Interests}
This independent research received no external or institutional funding. The authors have no
relevant financial or nonfinancial interests to disclose.

\FloatBarrier
{\small

}
\end{document}